\documentclass{article}

\PassOptionsToPackage{numbers, compress}{natbib}

\usepackage[final]{neurips_2019}

\usepackage[utf8]{inputenc} %
\usepackage[T1]{fontenc}    %
\usepackage{hyperref}       %
\usepackage{url}            %
\usepackage{booktabs}       %
\usepackage{amsfonts}       %
\usepackage{amsmath}
\usepackage{amsthm}
\usepackage{nicefrac}       %
\usepackage{microtype}      %
\usepackage{graphicx}
\usepackage{amssymb}
\usepackage{thmtools, thm-restate}
\usepackage[ruled,vlined,linesnumbered]{algorithm2e}
\newcommand{\eqdef}{\buildrel \mbox{\tiny\textrm{def}} \over =}
\newcommand{\IGApprox}{\widetilde{IG}}
\usepackage{xcolor}
\definecolor{COCcolor}{HTML}{1F77B4}
\definecolor{WLcolor}{HTML}{FDB863}
\definecolor{SLcolor}{HTML}{E66101}
\definecolor{WRcolor}{HTML}{B2ABD2}
\definecolor{SRcolor}{HTML}{5E3C99}
\definecolor{DiffAIGreen}{HTML}{078307}

\usepackage{subcaption}

\usepackage{prettyref}
\newcommand{\pref}{\prettyref}
\newrefformat{thm}{Theorem~\ref{#1}}
\newrefformat{lem}{Lemma~\ref{#1}}
\newrefformat{cha}{Chapter~\ref{#1}}
\newrefformat{sec}{Section~\ref{#1}}
\newrefformat{tab}{Table~\ref{#1}}
\newrefformat{fig}{Figure~\ref{#1}}
\newrefformat{alg}{Algorithm~\ref{#1}}
\newrefformat{exa}{Example~\ref{#1}}
\newrefformat{def}{Definition~\ref{#1}}
\newrefformat{li}{Line~\ref{#1}}
\newrefformat{eq}{Equation~\ref{#1}}
\newrefformat{app}{Appendix~\ref{#1}}

\newtheorem{definition}{Definition}

\newcommand{\problemname}{\textsc{ExactLine}}
\newcommand{\restr}[2]{#1{_{\restriction{#2}}}}
\newcommand{\exactline}[2]{\mathcal{P}\big(\restr{#1}{\overline{#2}}\big)}

\title{Computing Linear Restrictions of Neural Networks}

\author{
    Matthew~Sotoudeh \\ Department of Computer Science \\ University of
    California, Davis \\ Davis, CA 95616 \\ \texttt{masotoudeh@ucdavis.edu} \\
    \And Aditya~V.\ Thakur \\ Department of Computer Science \\ University of
    California, Davis \\ Davis, CA 95616 \\ \texttt{avthakur@ucdavis.edu} \\ }

\begin{document}

\maketitle

\begin{abstract} A linear restriction of a function is the same function with
    its domain restricted to points on a given line. This paper addresses the
    problem of computing a succinct representation for a linear restriction of
    a piecewise-linear neural network. This primitive, which we call
    $\problemname$, allows us to exactly characterize the
    result of applying the network to all of the infinitely many points on a
    line.  In particular, $\problemname$ computes a partitioning of the given
    input line segment such that the network is affine on each partition. We
    present an efficient algorithm for computing $\problemname$ for networks
    that use ReLU, MaxPool, batch normalization, fully-connected,
    convolutional, and other layers, along with several applications.  First,
    we show how to exactly determine decision boundaries of an ACAS Xu neural
    network, providing significantly improved confidence in the results
    compared to prior work that sampled finitely many points in the input
    space. Next, we demonstrate how to exactly compute integrated gradients,
    which are commonly used for neural network attributions, allowing us to
    show that the prior heuristic-based methods had relative errors of 25-45\%
    and show that a better sampling method can achieve higher accuracy with
    less computation.  Finally, we use $\problemname$ to empirically falsify
    the core assumption behind a well-known hypothesis about adversarial
    examples, and in the process identify interesting properties of
    adversarially-trained networks.
\end{abstract}

\section{Introduction}
\label{sec:Introduction}

The past decade has seen the rise of deep neural networks
(DNNs)~\cite{Goodfellow:DeepLearning2016} to solve a variety of problems,
including image recognition~\cite{Szegedy:CVPR2016,Krizhevsky:CACM2017},
natural-language processing~\cite{BERT:CoRR2018}, and autonomous vehicle
control~\cite{julian2018deep}.
However, such models are difficult to meaningfully interpret and check for
correctness. Thus, researchers have tried to understand the behavior of such
networks. For instance, networks have been shown to be vulnerable to
\emph{adversarial examples}---inputs changed in a way imperceptible to humans
but resulting in a misclassification by the
network~\citep{Szegedy:ICLR2014,Goodfellow:ICLR2015,DeepFool:CVPR2016,carlini2018audio}--and
\emph{fooling examples}---inputs that are completely unrecognizable by humans
but classified with high confidence by DNNs~\citep{Nguyen:CVPR2015}. The
presence of such adversarial and fooling inputs as well as applications in
safety-critical systems has led to efforts to verify, certify, and improve
robustness of
DNNs~\cite{Bastani:NIPS2016,reluplex:CAV2017,Ehlers:ATVA2017,Huang:CAV2017,ai2:SP2018,Bunel:NIPS2018,Weng:ICML2018,Singh:POPL2019,Anderson:PLDI2019,pmlr-v89-croce19a}.
Orthogonal approaches help visualize the behavior of the
network~\citep{Zeiler:ECCV2014,Yosinski:CoRR2015,Bau:CVPR2017} and interpret
its
decisions~\citep{Baehrens:JMLR2010,Shrikumar:CoRR2016,Sundararajan:ICML2017,tcav:ICML2018,Guidotti:CSUR2018}.
Despite the tremendous progress, more needs to be done to help understand DNNs
and increase their adoption
\citep{Ching:RSIF2017,DBLP:journals/bib/MiottoWWJD18,Hosny:NatureReviews2018,Mendelson:AJR2019}.

In this paper, we present algorithms for computing the $\problemname$
primitive: given a piecewise-linear neural network (e.g.~composed of
convolutional and ReLU layers) and line in the input space $\overline{QR}$, we
partition $\overline{QR}$ such that the network is affine on each partition.
Thus, $\problemname$ precisely captures the behavior of the network for the
infinite set of points lying on the line between two points. In effect,
$\problemname$ computes a succinct representation for a linear restriction of a
piecewise-linear neural network; a \emph{linear restriction} of a function is
the same function with its domain restricted to points on a given line.
We present an efficient implementation of $\problemname$~(\pref{sec:Primitive})
for piecewise-linear neural networks, as well as examples of how $\problemname$
can be used to understand the behavior of DNNs.
In \pref{sec:ACAS} we consider a problem posed by \citet{ReluVal:Usenix2018},
viz., to determine the classification regions of ACAS Xu~\cite{julian2018deep},
an aircraft collision avoidance network, when linearly interpolating between
two input situations.  This characterization can, for instance, determine at
which precise distance from the ownship a nearby plane causes the network to
instruct a hard change in direction.
\pref{sec:IntegratedGradients} describes how $\problemname$ can be used to
exactly compute the \emph{integrated gradients}~\citep{Sundararajan:ICML2017},
a state-of-the-art network attribution method that until now has only been
approximated. We quantify the error of previous heuristics-based methods, and
find that they result in attributions with a relative error of 25-45\%.
Finally, we show that a different heuristic using trapezoidal rule can produce
significantly higher accuracy with fewer samples.
\pref{sec:CountingRegions} uses $\problemname$ to probe interesting properties
of the neighborhoods around test images. We empirically reject a fundamental
assumption behind the Linear Explanation of Adversarial
Examples~\citep{Goodfellow:ICLR2015} on multiple networks. Finally, our results
suggest that DiffAI-protected~\citep{diffai2018} neural networks exhibit
significantly less non-linearity in practice, which perhaps contributes to
their adversarial robustness.
We have made our source code available at
\href{https://doi.org/10.5281/zenodo.3520097}{https://doi.org/10.5281/zenodo.3520097}.

\section{The $\problemname$ Primitive}
\label{sec:Primitive}

Given a piecewise-linear neural network $f$ and two points $Q, R$ in the input
space of $f$, we consider the \emph{restriction of $f$ to $\overline{QR}$},
denoted $\restr{f}{\overline{QR}}$, which is identical to the function $f$
except that its input domain has been restricted to $\overline{QR}$. We now want
to find a \emph{succinct representation} for $\restr{f}{\overline{QR}}$ that we
can analyze more readily than the neural network corresponding to $f$.  In this
paper, we propose to use the $\problemname$ representation, which corresponds to
a \emph{linear partitioning} of $\restr{f}{\overline{QR}}$, defined below.

\begin{definition}
    \label{def:exactline}
    Given a function $f: A \to B$ and line segment $\overline{QR} \subseteq A$,
    a tuple $(P_1, P_2, P_3, \ldots, P_n)$ is a \emph{linear
    partitioning of $\restr{f}{\overline{QR}}$}, denoted $\exactline{f}{QR}$
    and referred to as ``$\problemname$ of $f$ over $\overline{QR}$,'' if: (1)
    $\{ \overline{P_iP_{i+1}} \mid 1 \leq i < n \}$ partitions $\overline{QR}$
    (except for overlap at endpoints); (2) $P_1 = Q$ and $P_n = R$; and (3) for
    all $1 \leq i < n$, there exists an affine map $A_i$ such that $f(x) =
    A_i(x)$ for all $x \in \overline{P_iP_{i+1}}$.
\end{definition}

In other words, we wish to \emph{partition $\overline{QR}$ into a set of pieces
where the action of $f$ on all points in each piece is affine.} Note that,
given $\exactline{f}{QR} = (P_1, \ldots, P_n)$, the corresponding affine
function for each partition $\overline{P_iP_{i+1}}$ can be determined by
recognizing that affine maps preserve ratios along lines. In other words, given
point $x = (1 - \alpha)P_i + \alpha P_{i+1}$ on linear partition
$\overline{P_iP_{i+1}}$, we have $f(x) = (1 - \alpha)f(P_i) + \alpha
f(P_{i+1})$. In this way, $\exactline{f}{QR}$ provides us a succinct and
precise representation for the behavior of $f$ on all points along
$\overline{QR}$.

Consider an illustrative DNN taking as input the age and income
of an individual and returning a loan-approval score and premium that should be
charged over a baseline amount:
\begin{equation}
    f(X = (x_0, x_1)) = \mathrm{ReLU}\left(
    A(X) \right),
    \text{ where }
    A(X) = 
    \begin{bmatrix}
        -1.7 & 1.0 \\
        2.0 & -1.3 \\
    \end{bmatrix}
    X +
    \begin{bmatrix}
        3 \\
        3 \\
    \end{bmatrix}
    \label{eq:loan_net}
\end{equation}
Suppose an individual of $20$ years old making $\$30$k/year ($Q = (20,
30)$) predicts that their earnings will increase linearly every year until they
reach $30$ years old and are making $\$50$k/year ($R = (30, 50)$). We wish to
understand how they will be classified by this system over these 10 years.
We can use $\problemname$ (\pref{def:exactline}) to compute $\exactline{f}{QR}
= (P_1 = Q, P_2 = (23.\overline{3}, 36.\overline{6}),$ $P_3 = (26.\overline{6},
43.\overline{3}),$ $P_4 = R)$, where $\restr{f}{QR}$ is \emph{exactly}
described by the following piecewise-linear function
(\pref{fig:transformer_mlf}):
\begin{equation}
    \restr{f}{\overline{QR}}(x) =
    \small
    \begin{cases}
        \begin{bmatrix}
            0 & 0 \\
            2 & -1.3
        \end{bmatrix}x +
        \begin{bmatrix}
            0 \\
            3
        \end{bmatrix},
        & x \in \overline{QP_2} \\[10pt] 
        \begin{bmatrix}
            -1.7 & 1 \\
            2 & -1.3
        \end{bmatrix}
        x
        +
        \begin{bmatrix}
            3 \\
            3
        \end{bmatrix},
        & x \in \overline{P_2P_3} \\[10pt]
        \begin{bmatrix}
            -1.7 & 1 \\
            0 & 0
        \end{bmatrix}
        x
        +
        \begin{bmatrix}
            3 \\
            0
        \end{bmatrix},
        & x \in \overline{P_3R}
    \end{cases}
    \label{eq:loan_net_explicit}
\end{equation}

\begin{figure}[t]\centering
    \includegraphics[width=.8\linewidth]{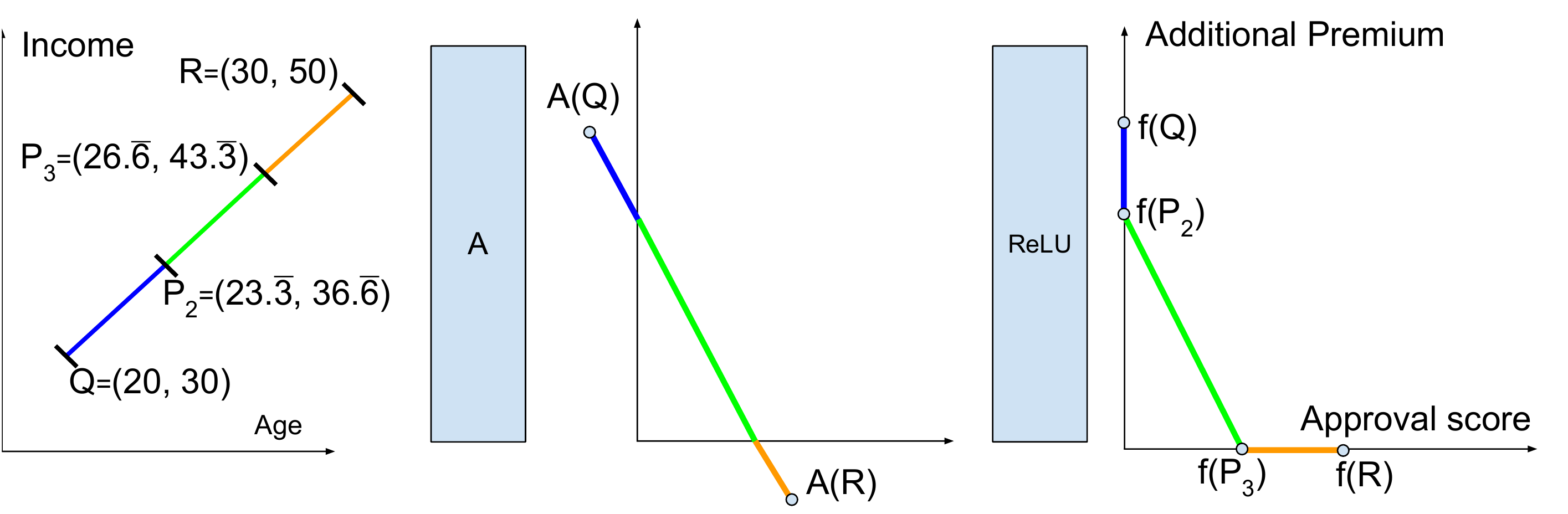}
    \caption{Computing the linear restriction of $f$ (\pref{eq:loan_net}) using
    $\problemname$. The input line segment $\overline{QR}$ is divided into
    three linear partitions such that the transformation from input space to
    output space (left plot to right plot) is affine
    (\pref{eq:loan_net_explicit}). Tick marks (on the left) are used in figures
    throughout this paper to indicate the partition endpoints $(P_1, P_2,
    \ldots, P_n)$.}
    \label{fig:transformer_mlf}
\end{figure}

\paragraph{Other Network Analysis Techniques}
Compared to prior work, our solution to $\problemname$ presents an interesting
and unique design point in the space of neural-network analysis. Approaches
such as~\citep{reluplex:CAV2017,Xiang:arxiv2017,Bunel:NIPS2018} are precise,
but exponential time because they work over the entire input domain. On another
side of the design space, approaches such as those used
in~\citep{ai2:SP2018,Xiang:arxiv2017,Xiang:ACC2018,Weng:ICML2018,Dutta:NFM2018,ReluVal:Usenix2018}
are significantly faster while still working over the full-dimensional input
space, but accomplish this by trading analysis precision for speed. This
trade-off between speed and precision is particularly well-illuminated
by~\citep{jordan2019provable}, which monotonically refines its analysis when
given more time. In contrast, the key observation underlying our work is that
we can perform \emph{both} an efficient (worst-case polynomial time for a fixed
number of layers) and precise analysis by \emph{restricting the input domain}
to be one dimensional (a line). This insight opens a new dimension to the
discussion of network analysis tools, showing that \emph{dimensionality} can be
traded for significant gains in \emph{both} precision and efficiency (as
opposed to prior work which has explored the tradeoff primarily along the
precision and efficiency axes under the assumption of high-dimensional input
regions). \citet{hanin2019complexity} similarly considers one-dimensional input
spaces, but the paper is focused on a number of theoretical properties and does
not focus on the algorithm used in their empirical results.

\paragraph{Algorithm}
We will first discuss computation of $\problemname$ for individual layers.
Note that by definition,
$\problemname$ for affine layers does not introduce any new linear partitions.
This is captured by~\pref{thm:affine-exactline} (proved
in~\pref{app:affine-exactline}) below:
\begin{restatable}{theorem}{ThmAffineExactline}
    \label{thm:affine-exactline}
    For any affine function $A : X \to Y$ and line segment $\overline{QR}
    \subset X$, the following is a suitable linear partitioning
    (\pref{def:exactline}): $\exactline{A}{QR} = (Q, R)$.
\end{restatable}
The following theorem (proved in~\pref{app:relu-exactline}) presents a method
of computing $\exactline{\mathrm{ReLU}}{QR}$.
\begin{restatable}{theorem}{ThmReluExactline}
    \label{thm:relu-exactline}
    Given a line segment $\overline{QR}$ in $d$ dimensions and a rectified
    linear layer $\mathrm{ReLU}(x) = (\mathrm{max}(x_1, 0), \ldots,
    \mathrm{max}(x_d, 0))$, the following is a suitable linear partitioning
    (\pref{def:exactline}):
    \begin{equation}
        \exactline{\mathrm{ReLU}}{QR} =
        \mathrm{sorted}
        \left(
            \left(
                \{ Q, R \}
                \cup
                \{ Q + \alpha(R - Q) \mid \alpha \in D \}
            \right)
            \cap
            \overline{QR}
        \right),
        \label{eq:alg}
    \end{equation}
    where $D = \left\{ -Q_i / (R_i - Q_i) \mid 1 \leq i \leq d \right\}$, $V_i$
    is the $i$th component of vector $V$, and $\mathrm{sorted}$ returns a tuple
    of the points sorted by distance from $Q$.
\end{restatable}

The essential insight is that we can ``follow'' the line until an orthant
boundary is reached, at which point a new linear region begins. To that end,
each number in $D$ represents a ratio between $Q$ and $R$ at which
$\overline{QR}$ crosses an orthant boundary.  Notably, $D$ actually computes
such ``crossing ratios'' for the \emph{unbounded} line $\overline{QR}$, hence
intersecting the generated endpoints with $\overline{QR}$ in~\pref{eq:alg}.

An analogous algorithm for MaxPool is presented
in~\pref{app:maxpool-exactline}; the intuition is to follow the line until the
maximum in any window changes. When a ReLU layer is followed by a MaxPool layer
(or vice-versa), the ``fused'' algorithm described
in~\pref{app:relumaxpool-exactline} can improve efficiency significantly.
More generally, the algorithm described in~\pref{app:pwl-exactline} can compute
$\problemname$ for any piecewise-linear function.

Finally, in practice we want to compute $\exactline{f}{QR}$ for entire
\emph{neural networks} (i.e. sequential compositions of layers), not just
individual layers (as we have demonstrated above). The next theorem shows that,
as long as one can compute $\exactline{L_i}{MN}$ for each individual layer
$L_i$ and arbitrary line segment $\overline{MN}$, then these algorithms can be
\emph{composed} to compute $\exactline{f}{QR}$ \emph{for the entire network}.
\begin{restatable}{theorem}{ThmMultiLayer}
    \label{thm:multi-layer}
    Given any piecewise-linear functions $f, g, h$ such that $f = h \circ g$
    along with a line segment $\overline{QR}$ where $g(R) \neq g(Q)$ and
    $\exactline{g}{QR} = (P_1, P_2, \ldots, P_n)$ is $\problemname$
    applied to $g$ over $\overline{QR}$, the following holds:
    \[ \exactline{f}{QR} =
        \mathrm{sorted}\left(
        \bigcup_{i=1}^{n - 1}{
            \left\{ P_i + \frac{y - g(P_i)}{g(P_{i+1}) - g(P_i)}\times (P_{i+1} - P_i) \mid
               y \in \exactline{h}{g(P_i)g(P_{i+1})}
            \right\}}
        \right)
    \]
    where $\mathrm{sorted}$ returns a tuple of the points sorted by distance
    from $Q$.
\end{restatable}
The key insight is that we can first compute $\problemname$ for the first
layer, i.e. $\exactline{L_1}{QR} = (P^1_1, P^1_2, \ldots, P^1_n)$, then we can
continue computing $\problemname$ for the rest of the network \emph{within each
of the partitions $\overline{P^1_iP^1_{i+1}}$ individually}.

In~\pref{app:exactline-runtime} we show that, over arbitrarily many affine
layers, $l$ ReLU layers each with $d$ units, and $m$ MaxPool or MaxPool + ReLU
layers with $w$ windows each of size $s$, at most $\mathrm{O}((d + ws)^{l +
m})$ segments may be produced. If only ReLU and affine layers are used, at most
$\mathrm{O}(d^l)$ segments may be produced.
Notably, this is a significant improvement over the $\mathrm{O}((2^d)^l)$
upper-bound and $\mathrm{\Omega}(l\cdot(2^d))$ lower-bound
of~\citet{Xiang:arxiv2017}.  One major reason for our improvement is that we
particularly consider one-dimensional input \emph{lines} as opposed to
arbitrary polytopes. Lines represent a particularly efficient special case as
they are efficiently representable (by their endpoints) and, being
one-dimensional, are not subject to the combinatorial blow-up faced by
transforming larger input regions.  Furthermore, in practice, we have found
that the majority of ReLU nodes are ``stable'', and the actual number of
segments remains tractable; this algorithm for $\problemname$ often executes in
a matter of seconds for networks with over $60,000$ units (whereas the
algorithm of~\citet{Xiang:arxiv2017} runs in at least exponential
$\mathrm{O}(l\cdot(2^d))$ time regardless of the input region as it relies on
trivially considering \emph{all possible} orthants).

\section{Characterizing Decision Boundaries for ACAS Xu}
\label{sec:ACAS}

\begin{figure}
    \centering
    Legend:
    \textcolor{COCcolor}{\rule[.2\baselineskip]{1em}{2pt}}
    Clear-of-Conflict,
    \textcolor{WRcolor}{\rule[.2\baselineskip]{1em}{2pt}} Weak Right, 
    \textcolor{SRcolor}{\rule[.2\baselineskip]{1em}{2pt}} Strong Right, 
    \textcolor{SLcolor}{\rule[.2\baselineskip]{1em}{2pt}} Strong Left, 
    \textcolor{WLcolor}{\rule[.2\baselineskip]{1em}{2pt}} Weak Left.
    \\
    \vspace{1em}
    \begin{subfigure}{.3\linewidth}
        \includegraphics[width=\linewidth]{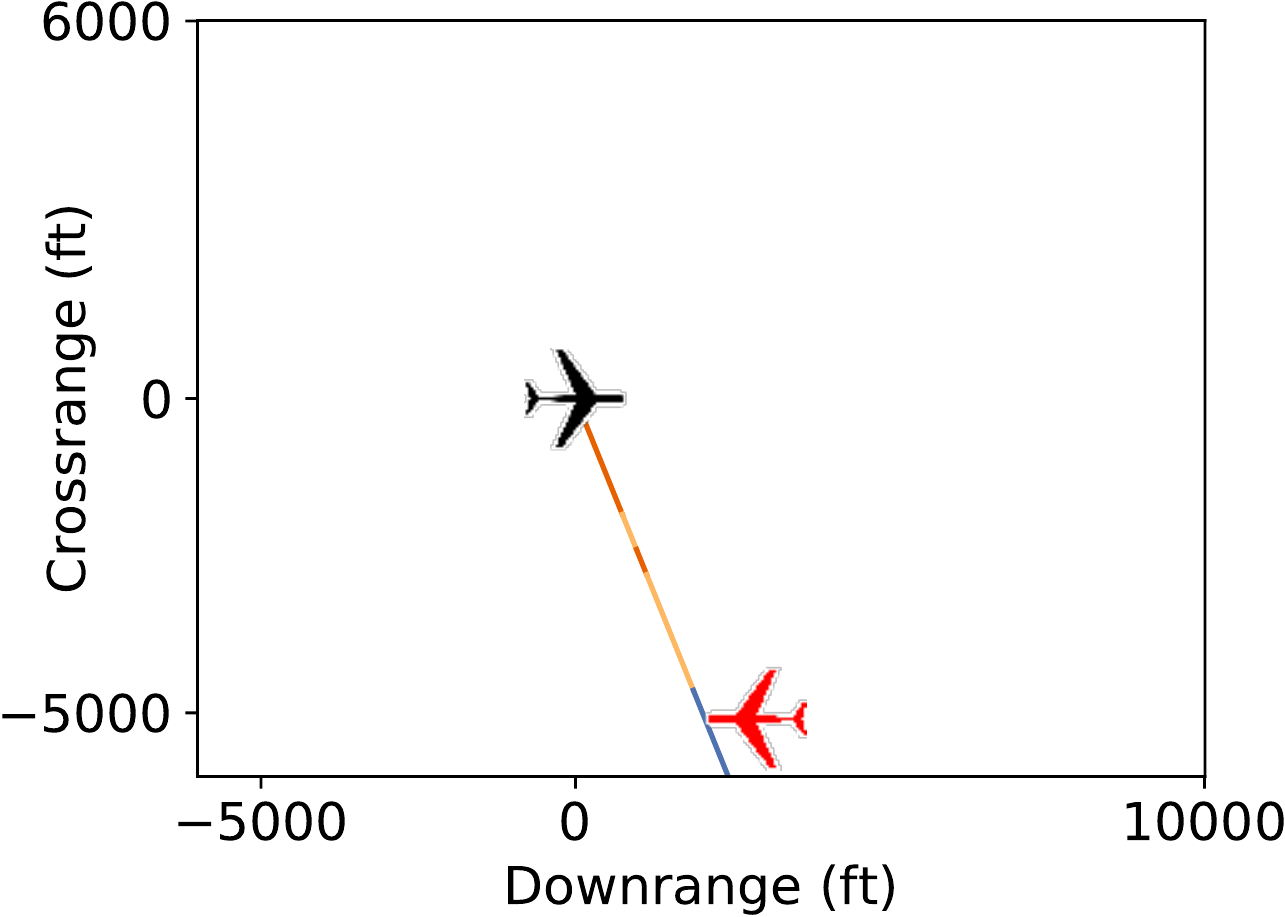}
        \caption{Single line varying $\rho$}
        \label{fig:acas-single-distance}
    \end{subfigure}
    \begin{subfigure}{.3\linewidth}
        \includegraphics[width=\linewidth]{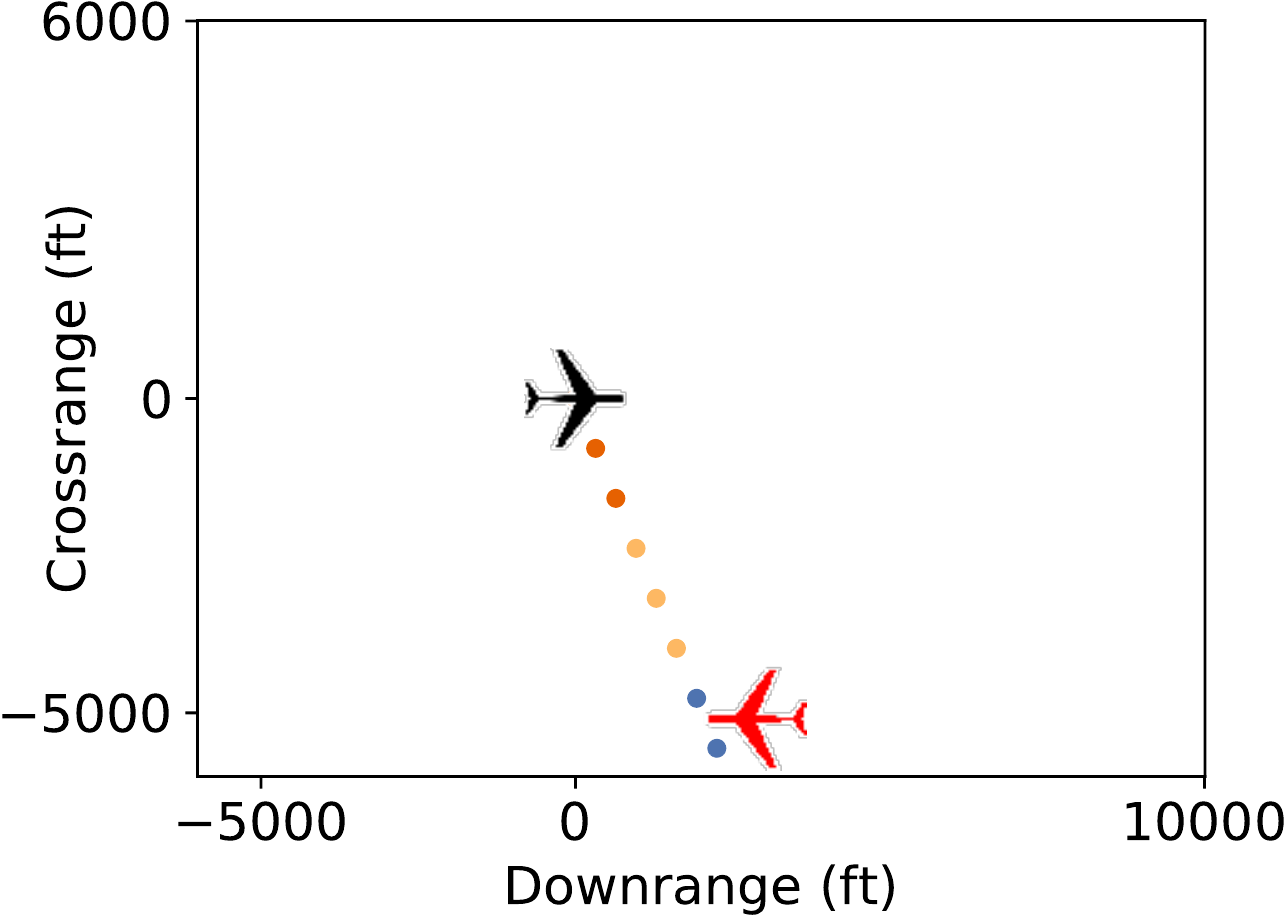}
        \caption{Sampling different $\rho$s}
        \label{fig:acas-single-distance-sample}
    \end{subfigure}
    \begin{subfigure}{.3\linewidth}
        \includegraphics[width=\linewidth]{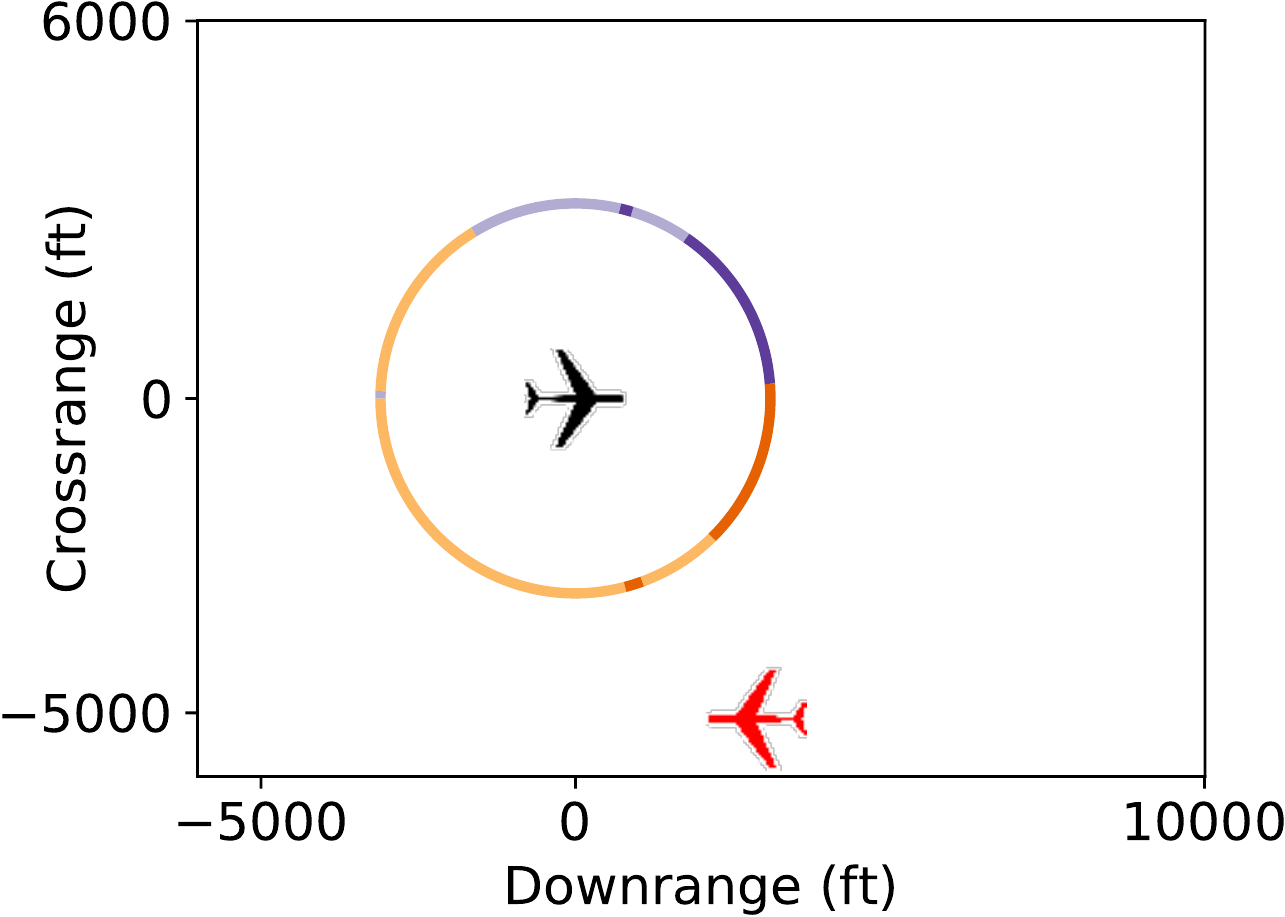}
        \caption{Single line varying $\theta$}
        \label{fig:acas-single-theta}
    \end{subfigure}
    \begin{subfigure}{.3\linewidth}
        \includegraphics[width=\linewidth]{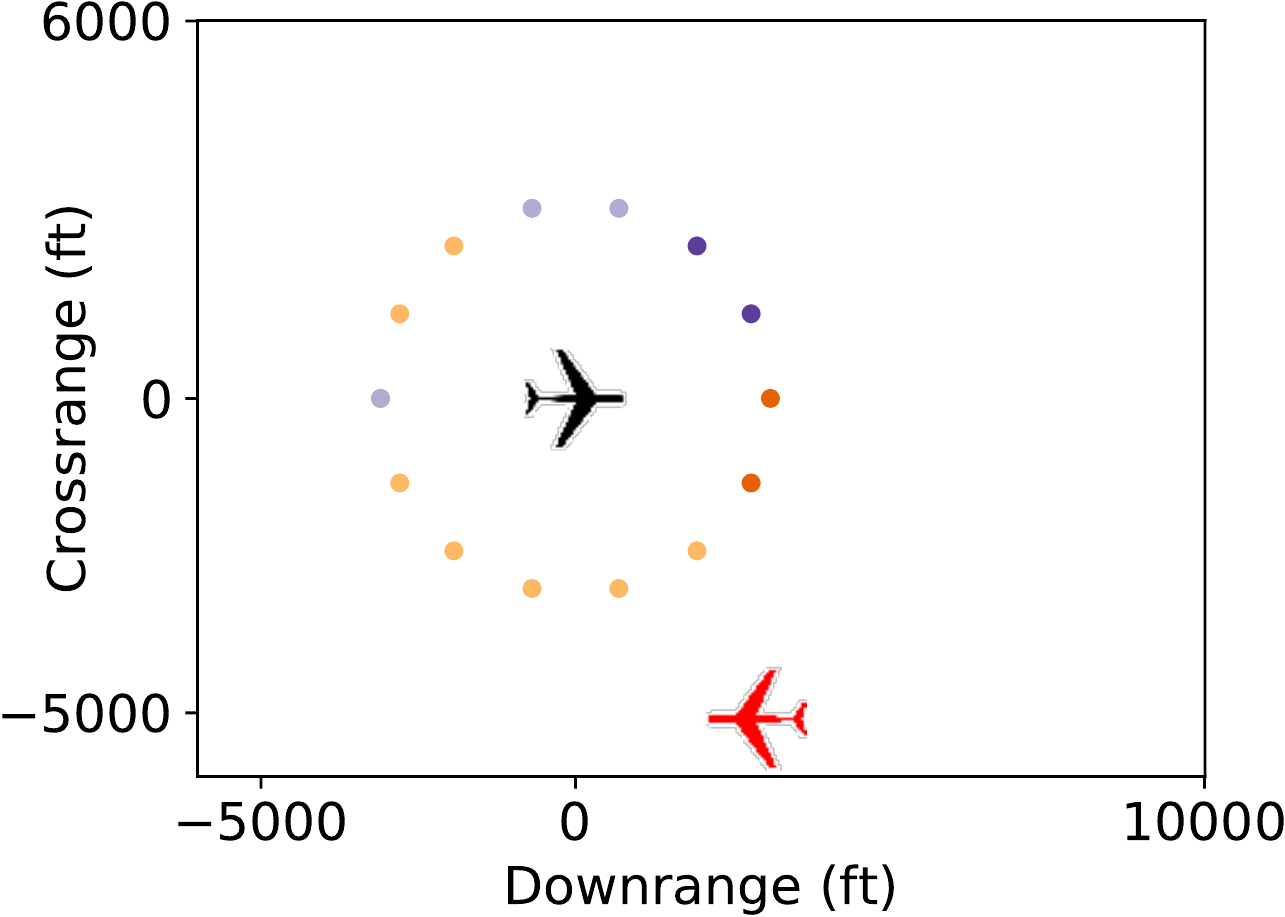}
        \caption{Sampling different $\theta$s}
        \label{fig:acas-single-theta-sample}
    \end{subfigure}
    \begin{subfigure}{.3\linewidth}
        \includegraphics[width=\linewidth]{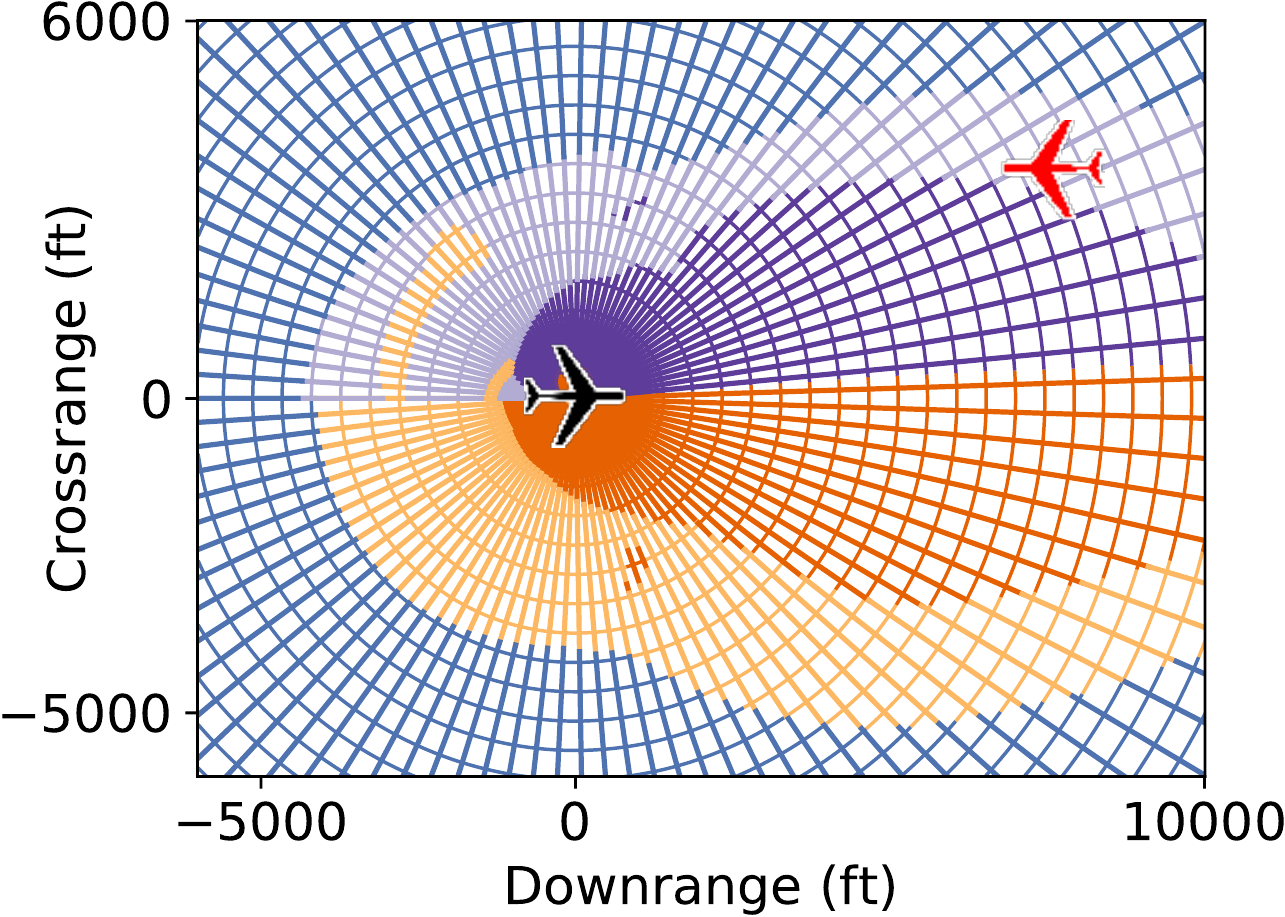}
        \caption{Combination of lines varying $\rho$ and lines varying
        $\theta$}
        \label{fig:acas-overlapping}
    \end{subfigure}
    \begin{subfigure}{.3\linewidth}
        \includegraphics[width=\linewidth]{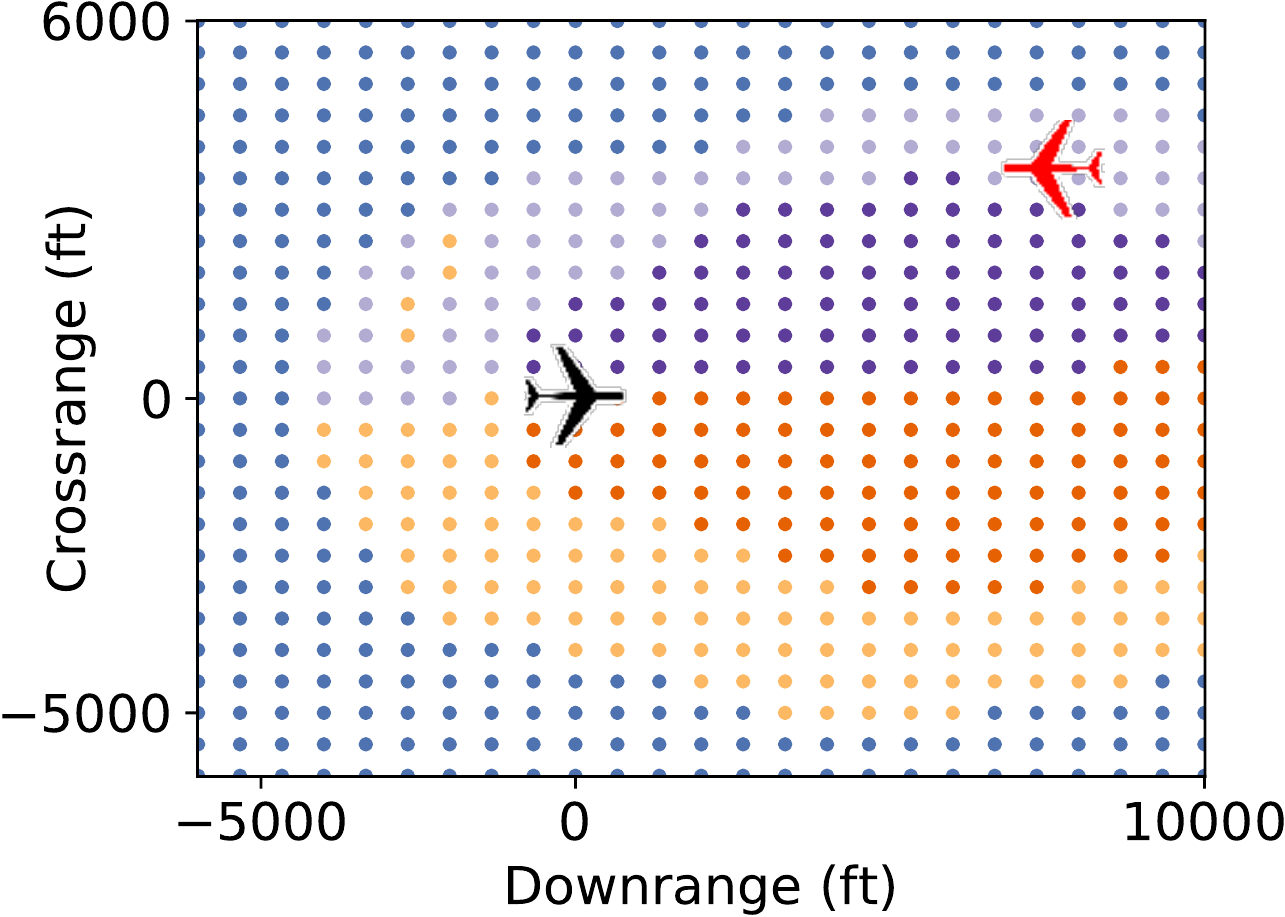}
        \caption{Sampling finitely many points}
        \label{fig:acas-sample}
    \end{subfigure}
    \caption{(a)--(d) Understanding the decision boundary of an ACAS Xu
    aircraft avoidance network along individual lines using $\problemname$
    ((a), (c)) and finite sampling ((b), (d)). In the $\problemname$
    visualizations there is a clear region of ``strong left'' in a region that
    is otherwise ``weak left'' that does not show in the sampling plots due to
    the sampling density chosen. In practice, it is not clear what sampling
    density to choose, thus the resulting plots can be inaccurate and/or
    misleading. (e)--(f) Computing the decision boundaries among multiple lines
    and plotting on the same graph. Using $\problemname$ to sample
    \emph{infinitely many} points provides more confidence in the
    interpretation of the decision boundaries. Compare to similar figures
    in~\citet{julian2018deep,reluplex:CAV2017}.
    }
    \label{fig:acas-1}
\end{figure}

The first application of $\problemname$ we consider is that of understanding
the decision boundaries of a neural network over some infinite set of inputs.
As a motivating example, we consider the ACAS Xu network trained by
\citet{julian2018deep} to determine what action an aircraft (the ``ownship'')
should take in order to avoid a collision with an intruder. After training such
a network, one usually wishes to probe and visualize the recommendations of the
network. This is desirable, for example, to determine at what distance from the
ownship an intruder causes the system to suggest a strong change in heading, or
to ensure that distance is roughly the same regardless of which side the
intruder approaches.

The simplest approach, shown in \pref{fig:acas-sample} and currently the
standard in prior work, is to consider a (finite) set of possible input
situations (samples) and see how the network reacts to each of them. This can
help one get an overall idea of how the network behaves.  For example, in
\pref{fig:acas-sample}, we can see that the network has a mostly symmetric
output, usually advising the plane to turn away from the intruder when
sufficiently close.
Although sampling in this way gives human viewers an intuitive and meaningful
way of understanding the network's behavior, it is severely limited because it
relies on sampling \textit{finitely many points} from a (practically)
\textit{infinite input space}. Thus, there is a significant chance that some
interesting or dangerous behavior of the network may be exposed with more
samples.

By contrast, the $\problemname$ primitive can be used to \textit{exactly}
determine the output of the network at \textit{all} of the \emph{infinitely
many} points on a line in the input region. For example,
in~\pref{fig:acas-single-distance}, we have used $\problemname$ to visualize a
particular head-on collision scenario where we vary the distance of the
intruder (specified in polar coordinates $(\rho, \theta)$) with respect to the
ownship (always at $(0, 0)$).  Notably, there is a region of ``Strong Left'' in
a region of the line that is otherwise entirely ``Weak Left`` that shows up
in~\pref{fig:acas-single-distance} (the $\problemname$ method) but not
in~\pref{fig:acas-single-distance-sample} (the sampling method). We can do this
for lines varying the $\theta$ parameter instead of $\rho$, result
in~\pref{fig:acas-single-theta} and~\pref{fig:acas-single-theta-sample}.
Finally, repeating this process for many lines and overlapping them on the same
graph produces a detailed ``grid'' as shown in~\pref{fig:acas-overlapping}.

\pref{fig:acas-overlapping} also shows a number of interesting and potentially
dangerous behaviors. For example, there is a significant region behind the plane
where an intruder on the left may cause the ownship to make a weak left turn
\textit{towards} the intruder, an unexpected and asymmetric behavior.
Furthermore, there are clear regions of strong left/right where the network
otherwise advises weak left/right. Meanwhile, in~\pref{fig:acas-sample}, we see
that the sampling density used is too low to notice the majority of this
behavior. In practice, it is not clear what sampling density should be taken to
ensure all potentially-dangerous behaviors are caught, which is unacceptable for safety-critical
systems such as aircraft collision avoidance.

\noindent\textbf{Takeaways.}  $\problemname$ can be used to visualize the network's output
on \emph{infinite subsets of the input space}, significantly improving the
confidence one can have in the resulting visualization and in the safety and
accuracy of the model being visualized.

\noindent\textbf{Future Work.} One particular area of future work in this direction is
using $\problemname$ to assist in network verification tools such
as~\citet{reluplex:CAV2017} and~\citet{ai2:SP2018}. For example, the
relatively-fast $\problemname$ could be used to check infinite subsets of the
input space for counter-examples (which can then be returned immediately)
before calling the more-expensive complete verification tools.

\section{Exact Computation of Integrated Gradients}
\label{sec:IntegratedGradients}

Integrated Gradients (IG)~\cite{Sundararajan:ICML2017} is a method of
attributing the prediction of a DNN to its input features.  Suppose function $F
: \mathbb{R}^n \to [0,1]$ defines the network. The \emph{integrated gradient}
along the $i^{th}$ dimension for an input $x= (x_1,\ldots, x_n) \in
\mathbb{R}^n$ and baseline $x' \in \mathbb{R}^n$ is defined as:
\begin{equation}
    \small
    IG_i(x)  \eqdef  (x_i - x_i') \times \int_{\alpha=0}^{1}\frac{ \partial F(x' + \alpha \times (x-x') )}{\partial x_i}d\alpha 
    \label{eq:IG}
\end{equation}
Thus, the integrated gradient along all input dimensions is the integral of the
gradient computed on all points on the straightline path from the baseline $x'$
to the input $x$. In prior work it was not known how to exactly compute this
integral for complex networks, so it was approximated using the left-Riemann
summation of the gradients at $m$ uniformly sampled points along the
straightline path from $x'$ to $x$:
\begin{equation}
    \small
\IGApprox_i^m \eqdef (x_i - x_i') \times \sum_{0 \leq k < m} \frac{ \partial F(x' + \frac{k}{m} \times (x-x') )}{\partial x_i}\times \frac{1}{m} 
\label{eq:IGApprox}
\end{equation}
The number of samples $m$ determines the quality of this approximation. Let
$\widetilde{m}$ denote the number of samples that is large enough to ensure
that $\sum_{i=1}^{n} \IGApprox_i^{\widetilde{m}} \approx F(x) - F(x')$. This is
the recommended number of samples suggested by \citet{Sundararajan:ICML2017}.
In practice, $\widetilde{m}$ can range between 20 and 1000
\cite{Sundararajan:ICML2017,Preuer:Corr2019}.

While the (exact) IG described by \pref{eq:IG} satisfies properties such as
completeness~\cite[\S3]{Sundararajan:ICML2017} and
sensitivity~\cite[\S2.1]{Sundararajan:ICML2017}, the approximation computed in
practice using \pref{eq:IGApprox} need not.

\begin{figure}[t] \centering
    \includegraphics[width=.6\linewidth]{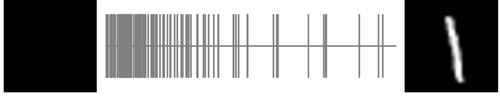}
    \caption{``Integrated Gradients'' is a powerful method of neural network
    attribution.  IG relies on computing the integral of the gradients of the
    network at all points linearly interpolated between two images (as shown
    above), however previous work has only been able to \emph{approximate} the
    true IG, casting uncertainty on the results. Within each partition
    identified by $\problemname$ (delineated by vertical lines) the gradient is
    constant, so computing the \emph{exact} IG is possible for the first time.}
    \label{fig:igrad1}
\end{figure}

The integral in \pref{eq:IG} can be computed exactly by adding an additional
condition to the definition of $\problemname$: that the gradient of the network
within each partition is constant. It turns out that this stricter definition
is met by all of the algorithms we have discussed so far, a fact we discuss in
more detail in~\pref{app:constgrad}. For ReLU layers, for example, because the
network acts like a single affine map within each orthant, and we split the
line such that each partition is entirely contained within an orthant, the
network's gradient is constant within each orthant (and thus along each
$\problemname$ partition). This is stated formally by~\pref{thm:relu-constgrad}
and proved in~\pref{app:constgrad}:

\begin{restatable}{theorem}{ThmReluConstgrad}
    \label{thm:relu-constgrad}
    For any network $f$ with nonlinearities introduced only by ReLU functions
    and $\exactline{f}{QR} = (P_1, P_2, \ldots, P_n)$ computed according
    to~\pref{eq:alg}, the gradient of $f$ with respect to its input vector $x$,
    i.e. $\nabla f(x)$, is constant within each linear partition
    $\overline{P_iP_{i+1}}$.
\end{restatable}

This allows us to exactly compute the IG of each individual partition $r$
($RIG^r_i$) by finding the gradient at any point in that partition and
multiplying by the width of the partition:
\begin{equation}
    \small
    RIG^{r}_i(x) \eqdef (P_{(r + 1)i} - P_{ri}) \times
    \frac{\partial F(0.5\times (P_{ri} + P_{(r+1)i}))}{\partial x_i}
    \label{eq:RIG}
\end{equation}
Compared to \pref{eq:IG}, \pref{eq:RIG} computes the IG for partition for
$\overline{P_rP_{r+1}}$ by replacing the integral with a single term
(arbitrarily choosing $\alpha = 0.5$, i.e. the midpoint) because the gradient
is uniform along $\overline{P_rP_{r+1}}$.
The exact IG of $\overline{x'x}$ is the sum of the IGs of each partition:
\begin{equation} \small
IG_i(x) = \sum_{r = 1}^{n} RIG^{r}_i(x)
\label{eq:exactIG}
\end{equation}

\begin{table}[t]
    \begin{minipage}{.33\linewidth}\centering
        \caption{Mean relative error for approximate IG (using $\widetilde m$) 
        compared to exact IG (using $\problemname$)  on CIFAR10. The approximation error is
        surprisingly high.}
        \label{tab:ig-error}
        \small
    \begin{tabular}{@{}ll@{}} \toprule 
               & Error (\%) \\ \midrule 
    convsmall  & 24.95      \\
    convmedium & 24.05      \\
    convbig    & 45.34      \\ \bottomrule
    \end{tabular}
    \end{minipage}
    \hfill
    \begin{minipage}{.6\linewidth}\centering
        \caption{Average number of samples needed by different IG approximations
        to reach $5\%$ relative error w.r.t.\ exact IG (using $\problemname$). 
        Outliers requiring over $1,000$ samples were not considered. Using
        trapezoidal rule instead of a left-sum can cause large gains in
        accuracy and performance.}
        \label{tab:ig-samples}
        \small
    \begin{tabular}{@{}lllll@{}} \toprule
        & Exact & \multicolumn{3}{c}{ Approximate} \\
        \cmidrule{3-5}
               &          & left   & right  & trap.  \\ \midrule
    convsmall  & 2582.74  & 136.74 & 139.40 & 91.67  \\
    convmedium & 3578.89  & 150.31 & 147.59 & 91.57  \\
    convbig    & 14064.65 & 269.43 & 278.20 & 222.79 \\ \bottomrule
    \end{tabular}
    \end{minipage}
\end{table}
    
\noindent\textbf{Empirical Results.}
A prime interest of ours was to determine
the accuracy of the existing sampling method.
On three different CIFAR10 networks \cite{ERAN}, we took each image in
the test set and computed the exact IG against a black baseline using
\pref{eq:exactIG}. We then found $\widetilde{m}$ and computed the mean relative
error between the exact IG and the approximate one. As shown in
\pref{tab:ig-error}, the approximate IG has an error of $25-45\%$. This is a
concerning result, indicating that the existing use of IG in practice may be
misleading. Notably, without $\problemname$, there would be no way of realizing
this issue, as this analysis relies on computing the exact IG to compare with.

For many smaller networks considered, we have found that computing the exact IG
is relatively fast (i.e., at most a few seconds) and would recommend doing so
for situations where it is feasible. However, in some cases the number of
gradients that would have to be taken to compute exact IG is high (see Column
2, \pref{tab:ig-samples}). In such cases, we can use our exact computation to
understand how many uniform samples one should use to compute the approximate
IG to within $5\%$ of the exact IG. To that end, we performed a second
experiment by increasing the number of samples taken until the approximate IG
reached within $5\%$ error of the exact IG. In practice, we
found there were a number of ``lucky guess'' situations where taking, for
example, exactly $5$ samples led to very high accuracy, but taking $4$ or $6$
resulted in very poor accuracy. To minimize the impact of such outliers, we
also require that the result is within $5\%$ relative error when taking up to
$5$ more samples. In \pref{tab:ig-samples} the results of this experiment are
shown under the column ``left,'' where it becomes clear that relatively few
samples (compared to the number of actual linear regions under ``exact'') are
actually needed.
Thus, one can use
$\problemname$ on a test set to understand how many samples are
needed on average to get within a desired accuracy, then use that many samples
when approximating IG.

Finally, the left Riemann sum used by existing work in IG~\cite{Sundararajan:ICML2017} is only one of many possible sampling
methods; one could also use a right Riemann sum or the ``Trapezoidal Rule.''
 With $\problemname$ we can
compute the exact IG and quantify how much better or worse each approximation
method is.
To that end, the columns ``right'' and ``trap.'' in \pref{tab:ig-samples} show
the number of samples one must take to get consistently within $5\%$ of the
exact integrated gradients. Our results show the number of samples needed to
accurately approximate IG with trapezoidal rule is significantly less than
using left and right sums. Note that, because these networks are
piecewise-linear functions, it is possible for trapezoidal sampling to be worse
than left or right sampling, and in fact for all tested models there were
images where trapezoidal sampling was less accurate than left and right
sampling. This is why having the ability to compute the \emph{exact} IG is
important, to ensure that we can empirically quantify the error
involved in and justify choices between different heuristics.  Furthermore, we
find a per-image reduction in the number of samples needed of $20-40\%$ on
average. Thus, we also recommend that users of IG use a trapezoidal
approximation instead of the currently-standard left sum.

\noindent\textbf{Takeaways.} $\problemname$ is the first method for exactly computing
integrated gradients, as all other uses of the technique have sampled according
to \pref{eq:IGApprox}. From our results, we provide two suggestions to
practitioners looking to use IG: (1) Use $\problemname$ on a test set before
deployment to better understand how the number of samples taken relates to
approximation error, then choose sampling densities at runtime based on those
statistics and the desired accuracy; (2) Use the trapezoidal rule approximation
when approximating to get better accuracy with
fewer samples.

\noindent\textbf{Future Work.} 
$\problemname$ can be used to design new IG approximation using non-uniform
sampling. $\problemname$ can similarly be used to exactly compute other measures
involving path integrals such as neuron conductance~\citep{Dhamdhere:Arxiv2018},
a refinement of integrated gradients.

\section{Understanding Adversarial Examples}
\label{sec:CountingRegions}

\begin{figure}[t]\centering
    \includegraphics[width=.9\linewidth]{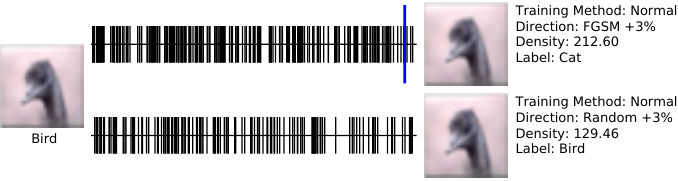}
    \caption{Comparing the density of linear partitions from a test image to
    FGSM and random baselines. The long blue line segment indicates the change
    in classification. (1) Even within a small perturbation of the input point,
    there is \emph{significant} non-linear behavior in the network. (2)
    Although the linear hypothesis for adversarial examples predicts that both
    normal inputs and their corresponding adversarial examples (``FGSM + 3\%'')
    will lie on the same linear region, we find in practice that, not only do
    they lie on different linear regions, but there is significantly
    \emph{more} non-linearity in the adversarial (FGSM) direction than a random
    direction.  This falsifies the fundamental assumption behind the linear
    explanation of adversarial examples.}
    \label{fig:count_samenet}
\end{figure}

\begin{figure}[t]\centering
    \includegraphics[width=.9\linewidth]{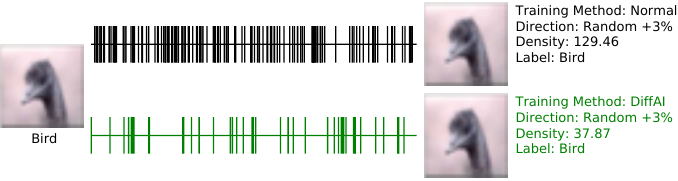}
    \caption{Comparing the density of linear partitions from a test image to
    random baselines for normal (in black) and DiffAI-protected networks (in
    {\color{DiffAIGreen}green}). Networks trained with the DiffAI robust-training
    scheme tend to exhibit significantly fewer non-linearities than their
    normally-trained counterparts.}
    \label{fig:count_diffnets}
\end{figure}

We use $\problemname$ to empirically investigate and falsify a
fundamental assumption behind a well-known hypothesis for the existence of adversarial
examples. In the process, using
$\problemname$, we discover a strong association between robustness and the
empirical \emph{linearity} of a network, which we believe may help spur future
work in understanding the source of adversarial examples.

\noindent\textbf{Empirically Testing the Linear Explanation of Adversarial Examples.}
We consider in this section one of the dominant hypotheses for the problem of
adversarial examples
(\citep{Szegedy:ICLR2014,Goodfellow:ICLR2015,DeepFool:CVPR2016}), first
proposed by~\citet{Goodfellow:ICLR2015} and termed the ``Linear Explanation of
Adversarial Examples.'' It makes a fundamental assumption that, when restricted
to the region around a particular input point, the output of the neural network
network is \emph{linear}, i.e. the tangent plane to the network's output at
that point exactly matches the network's true output.  Starting from this
assumption,~\citet{Goodfellow:ICLR2015} makes a theoretical argument concluding
that the phenomenon of adversarial examples results from the network being
``too linear.''
Although theoretical debate has ensued as to the merits of this
argument~\citep{tanay2016boundary}, the effectiveness of adversarials generated
according to this hypothesis (i.e., using the standard ``fast gradient sign
method'' or ``FGSM'') has sustained it as one of the most well-subscribed-to
theories as to the source of adversarial examples.  However, to our knowledge,
the fundamental assumption of the theory---that real images and their
adversarial counterparts lie on the same linear region of a neural
network---has not been rigorously validated empirically.

\setlength{\tabcolsep}{2pt}
\begin{table}[t] \centering
\caption{
    Density of FGSM line partitions divided by density of Random line
    partitions. FGSM directions tend to be significantly more dense
    than random directions, contradicting the well-known Linear
    Explanation of Adversarial Examples.
    {\small Mdn: Median, 25\%: 25\% percentile, 75\%: 75\% percentile } }
\begin{subtable}{.45\linewidth}  \centering
    \caption{MNIST}
    \label{tab:mnist-fgsm-hypothesis}
    {\small
\begin{tabular}{@{}llll@{}} \toprule
     & \multicolumn{3}{c}{FGSM/Random} \\
       \cmidrule{2-4}
    Training Method &  Mdn  & 25\% & 75\% \\ \midrule
    Normal          &  1.36 & 0.99 & 1.76 \\
    DiffAI          &  0.98 & 0.92 & 1.38 \\
    PGD             &  1.22 & 0.97 & 1.51 \\ \bottomrule
\end{tabular}
}%
\end{subtable}
\begin{subtable}{.45\linewidth}  \centering
    \caption{CIFAR10}
    \label{tab:cifar-fgsm-hypothesis}
    { \small
\begin{tabular}{@{}llll@{}} \toprule
     & \multicolumn{3}{c}{FGSM/Random} \\
       \cmidrule{2-4}
    Training Method & Mdn  & 25\% & 75\% \\ \midrule
    Normal          & 1.78 & 1.60 & 2.02 \\
    DiffAI          & 1.67 & 1.47 & 2.03 \\
    PGD             & 1.84 & 1.65 & 2.10 \\ \bottomrule
\end{tabular}
    }%
\end{subtable}
\label{tab:fgsm-hypothesis}
\vspace{-3ex}
\end{table}

To empirically validate this hypothesis, we looked
at the line between an image and its FGSM-perturbed counterpart (which is
classified differently by the network), then used $\problemname$ to quantify the
\emph{density} of each line (the number of linear partitions --- each delineated
by tick marks in~\pref{fig:count_samenet} --- divided by the distance between
the endpoints). If the underlying linearity assumption holds, we would expect
that both the real image and the perturbed image lie on the same linear
partition, thus we would not expect to see the line between them pass through
any other linear partitions.  In fact (top of~\pref{fig:count_samenet}), we find
that the adversarial image seems to lie across many linear partitions, directly
contradicting the fundamental assumption of~\citet{Goodfellow:ICLR2015}.
We also compared this line to the line between the real image and the real
image \emph{randomly} perturbed by the same amount (bottom
of~\pref{fig:count_samenet}). As shown in
\pref{tab:fgsm-hypothesis}, the FGSM direction seems to pass through
significantly \emph{more} linear partitions than a randomly chosen direction.
This result shows that, not only is the linearity assumption not met, but in
fact the opposite appears to be true: adversarial examples are associated with
\emph{unusually non-linear} directions of the network.

With these results in mind, we realized that the linearity assumption as
initially presented in~\citet{Goodfellow:ICLR2015} is stronger than necessary;
it need only be the case that the gradient is \emph{reasonably} constant across
the line, so that the tangent plane approximation used is still reasonably
accurate. To measure how well the tangent plane approximates the function
between normal and adversarial inputs, we compared the gradient taken at the
real image (i.e., used by FGSM) to the gradients at each of the intermediate
linear partitions, finding the mean error between each and averaging weighted
by size of the partition. If this number is near $0$, it implies that, although
many theoretically distinct linear partitions exist between the two points,
they have roughly the same ``slope'' and thus the tangent plane approximation
would be accurate and the Linearity Hypothesis may still be a worthwhile
explanation of the phenomenon of adversarial examples.  However, we find that
this number is larger than $250\%$ when averaged over all images tested,
implying that the tangent plane is \emph{not} a particularly good approximation
of the underlying function in these regions of input space and providing
further empirical evidence against the Linear Explanation of adversarial
examples.

\begin{table}[t] \centering
    \caption{
        Comparing density of lines when different training algorithms (normal,
        DiffAI, and PGD) are used. We report the mean of those ratios across
        all tested lines.  These results indicate that networks trained to be
        adversarially robust with DiffAI or PGD training methods tend to behave
        more linearly than non-robust models.
        {\small Mdn: Median, 25\%: 25\% percentile, 75\%: 75\% percentile } }
\begin{subtable}{.45\linewidth} \centering
\caption{MNIST}
\label{tab:mnist-robustness-hypothesis}
{\small
\begin{tabular}{@{}llllclll@{}} \toprule
    & \multicolumn{3}{c}{Normal/DiffAI} & \phantom{a} & \multicolumn{3}{c}{Normal/PGD}\\
      \cmidrule{2-4} \cmidrule{6-8} 
    Dir. &  Mdn & 25\% & 75\%  &&  Mdn  & 25\% & 75\% \\ \midrule 
    FGSM & 3.05 & 2.05 & 4.11  &&  0.88 & 0.66 & 1.14 \\
    Rand.& 2.50 & 1.67 & 3.00  &&  0.80 & 0.62 & 1.00 \\ \bottomrule
\end{tabular}
} %
\end{subtable}
\begin{subtable}{.45\linewidth} \centering
    \caption{CIFAR10}
    \label{tab:cifar-robustness-hypothesis}
    {\small
    \begin{tabular}{@{}llllclll@{}} \toprule
        & \multicolumn{3}{c}{Normal/DiffAI} & \phantom{a} & \multicolumn{3}{c}{Normal/PGD}\\
          \cmidrule{2-4} \cmidrule{6-8} 
        Dir.  & Mdn  & 25\% & 75\%  &&  Mdn   & 25\%  & 75\%\\ \midrule 
        FGSM  & 3.37 & 2.77 & 4.94  &&  1.48 & 1.22 & 1.67 \\
        Rand. & 3.43 & 2.78 & 4.42  &&  1.51 & 1.21 & 1.80 \\ \bottomrule
    \end{tabular}
    } %
\end{subtable}
\label{tab:robustness-hypothesis}
\end{table}
\setlength{\tabcolsep}{6pt}

\noindent\textbf{Characteristics of Adversarially-Trained Networks.}
We also noticed an unexpected trend in the previous experiment: networks
trained to be robust to adversarial perturbations (particularly DiffAI-trained
networks~\citep{diffai2018}) seemed to have \emph{significantly} fewer linear
partitions in all of the lines that we evaluated them on (see
\pref{fig:count_diffnets}).  Further experiments, summarized in
\pref{tab:robustness-hypothesis}, showed that networks trained to be
adversarially robust with PGD and especially DiffAI training methods exhibit up
to $5\times$ fewer linear partitions for the same change in input.  This
observation suggests that the neighborhoods around points in
adversarially-trained networks are ``flat'' (more linear).

\noindent\textbf{Takeaways.} Our results falsify the fundamental assumption
behind the well-known Linear Explanation for adversarial examples. Adversarial
training tends to make networks more linear.

\noindent\textbf{Future Work.} $\problemname$ can be used to investigate
adversarial robustness. Further investigation into why DiffAI-protected networks
tend to be more linear will help resolve the question (raised in this work) of
whether reduced density of linear partitions contributes to robustness, or
increased robustness results in fewer linear partitions (or if there is a third
important variable impacting both).

\section{Conclusion}
\label{sec:Conclusion}

We address the problem of computing a succinct representation of a linear
restriction of a neural network. We presented $\problemname$, a novel primitive
for the analysis of piecewise-linear deep neural networks. Our algorithm runs
in a matter of a few seconds on large convolutional and ReLU networks,
including ACAS Xu, MNIST, and CIFAR10. This allows us to investigate questions
about these networks, both shedding new light and raising new questions about
their behavior.

\subsubsection*{Acknowledgements}
We thank Nina Amenta, Yong Jae Lee, Cindy Rubio-Gonz\'alez, and Mukund
Sundararajan for their feedback and suggestions on this work.
This material is based upon work supported by AWS Cloud Credits for Research. 

\bibliographystyle{unsrtnat}
\bibliography{main}

\clearpage
\section*{Computing Linear Restrictions of Neural Networks: Supplemental}
\appendix
\appendix

\section{Specification of Evaluation Hardware}
\label{app:hw}
Although we do not claim particular performance results, we do point out that
all $\problemname$ uses in our experiments took only a matter of seconds on
commodity hardware (although in some cases the experiments themselves took a
few minutes, particularly when computing gradients for Section 4).

For reproducibility, all experimental data reported was run on Amazon EC2
c5.metal instances, using BenchExec~\cite{benchexec} to limit to 16 CPU cores
and 16 GB of memory. We have also run the results on commodity hardware, namely
an Intel Core i7-7820X CPU at 3.6GHz with 32GB of memory (both resources shared
with others simultaneously), for which the ``matter of seconds''
characterization above holds as well.

All experiments were only run on CPU, although we believe computing
$\problemname$ on GPUs is an important direction for future research on
significantly larger networks.

\section{Uniqueness of $\problemname$}
The \emph{smallest} tuple satisfying the requirements on $\exactline{f}{QR}$ is
unique (when it exists) for any given $f$ and $QR$, and any tuple satisfying
$\exactline{f}{QR}$ can be converted to this minimal tuple by removing any
endpoint $P_i$ which lies on $\overline{P_{i-1}P_{i+1}}$.  In the proceeding
text we will discuss methods for computing \emph{some} tuple satisfying
$\exactline{f}{QR}$; if the reader desires, they can use the procedure
mentioned in the previous sentence to reduce this to the unique smallest such
tuple. However, we note that the algorithms below \emph{usually} produce the
minimal tuple on real-world networks even without any reduction procedure due
to the high dimensionality of the functions involved.

\section{Runtime of $\problemname$ Algorithms}
\label{app:exactline-runtime}
We note that the algorithm corresponding to~\pref{thm:affine-exactline} runs on
a single line segment in constant time producing a single resulting line
segment. The algorithm corresponding to~\pref{eq:alg} runs on a single line
segment in $\mathrm{O}(d)$ time, producing at most $\mathrm{O}(d)$ new
segments. If $w$ is the number of windows and $s$ is the size of each window,
the algorithms for MaxPool and ReLU + MaxPool run in time $\mathrm{O}(ws^2)$
and produce $\mathrm{O}(ws)$ new line segments. Thus, using the algorithm
corresponding to~\pref{thm:multi-layer}, over arbitrarily many affine layers,
$l$ ReLU layers each with $d$ units, and $m$ MaxPool or MaxPool + ReLU layers
with $w$ windows each of size $s$, then at most $\mathrm{O}((d + ws)^{l + m})$
segments may be produced. If only $l$ ReLU and arbitrarily-many affine layers
are used, at most $\mathrm{O}(d^l)$ segments may be produced.

\section{$\problemname$ for Affine Layers}
\label{app:affine-exactline}

\ThmAffineExactline*

\begin{proof}
    By the definition of $\exactline{A}{QR}$ it suffices to show that $\{
        \overline{QR} \}$ partitions $\overline{QR}$ and produce an affine map
    $A'$ such that $A(x) = A'(x)$ for every $x \in \overline{QR}$.

    The first fact follows directly, as $\overline{QR} = \overline{QR}
    \implies \overline{QR} \subseteq \overline{QR}$ and every element of
    $\overline{QR}$ belongs to $\overline{QR}$.

    For the second requirement, we claim that $A' = A$ satisfies the desired
    property, as $A$ is affine and $A(x) = A(x)$ for all $x$ in general and in
    particular for all $x \in \overline{QR}$.
\end{proof}

\section{$\problemname$ for ReLU Layers}
\label{app:relu-exactline}

\ThmReluExactline*
\begin{proof}
    First, we define the ReLU function like so:
    \[
        \begin{bmatrix}
            \delta_{1\mathrm{sign}(x_1)} & 0 & \cdots & 0 \\
            0 & \delta_{1\mathrm{sign}(x_2)} & \cdots & 0 \\
            \vdots & \vdots & \ddots & \vdots \\
            0 & 0 & \cdots & \delta_{1\mathrm{sign}(x_d)} \\
        \end{bmatrix}
        \begin{bmatrix}
            x_1 \\
            x_2 \\
            \vdots \\
            x_d
        \end{bmatrix}
    \]
    where $\mathrm{sign}(x)$ returns $1$ if $x$ is positive and $0$ otherwise while
    $\delta_{ij}$ is the Kronecker delta.

    Now, it becomes clear that, as long as the signs of each $x_i$ are
    constant, the ReLU function is linear.

    We note that, over a Euclidean line segment $\overline{QR}$, we can
    parameterize $\overline{QR}$ as $\overline{QR}(\alpha) = Q + \alpha(R -
    Q)$. Considering the $i$th component, we have a linear relation
    $\overline{QR}_i(\alpha) = Q_i + \alpha(R_i - Q_i)$ which changes sign at
    most once, when $\overline{QR}_i(\alpha) = 0$ (because linear functions in
    $\mathbb{R}$ are continuous and monotonic). Thus, we can solve for the sign
    change of the $i$th dimension as:
    \begin{align*}
            &\overline{QR}_i(\alpha) = 0 \\
            \implies &0 = Q_i + \alpha(R_i - Q_i) \\
            \implies &\alpha = -\frac{Q_i}{R_i - Q_i}
        \end{align*}
    As we have restricted the function to $0 \leq \alpha \leq 1$, at any
    $\alpha$ within these bounds the sign of some component changes and the
    function acts non-linearly. Between any two such $\alpha$s, however, the
    signs of all components are constant, so the ReLU function is perfectly
    described by a linear map as shown above. Finally, we can solve for the
    endpoints corresponding to any such $\alpha$ using the parameterization
    $\overline{QR}(\alpha)$ defined above, resulting in the formula in the
    theorem.

    $Q, R$ are included to meet the partitioning definition, as the sign of
    some element may not be $0$ at the $Q, R$ endpoints.
\end{proof}

\section{$\problemname$ for MaxPool Layers}
\label{app:maxpool-exactline}

As discussed in the paper, although we do not use MaxPool layers in any of our
evaluated networks, we have developed and implemented an algorithm for
computing $\exactline{\mathrm{MaxPool}}{QR}$, which we present here. In
particular, we present $\exactline{\mathrm{MaxPoolWindow}}{QR}$, i.e. the
linear restriction for any given window.  $\exactline{\mathrm{MaxPool}}{MN}$
can be then be computed by separating each window $\overline{QR}$ from
$\overline{MN}$ and applying $\exactline{\mathrm{MaxPoolWindow}}{QR}$. Notably,
there may be duplicate endpoints (e.g. if there is overlap in the windows)
which can be handled by removing duplicates if desired.

\begin{algorithm}[H]
    \DontPrintSemicolon
    \SetKw{returnKw}{return}
    \SetKw{breakKw}{break}
    \KwIn{$\overline{QR}$, the line segment to restrict the MaxPoolWindow
    function to.}
    \KwOut{$\exactline{\mathrm{MaxPoolWindow}}{QR}$}
    $\mathcal{P} \gets [Q]$ \tcp*[l]{Begin an (ordered) list of points with one item, $Q$.}
    $\alpha \gets 0.0$ \tcp*[l]{Ratio along $\overline{QR}$ of the last endpoint in $\mathcal{P}$.}
    $m \gets \mathrm{argmax}(Q)$ \tcp*[l]{Maximum component of the last endpoint in $\mathcal{P}$.}
    \While{$m \neq \mathrm{argmax}(R)$}{
        $D \gets \frac{Q - Q_m}{(R_i - Q_i) - (R - Q)}$\;
        $A \gets \{ (D_i, i) \mid 1 \leq i \leq d \wedge \alpha < D_i < 1.0 \}$\;
        \lIf{$A = \emptyset$}{
            \breakKw
        }
        $(\alpha, m) \gets \mathrm{lexmin}(A)$ \tcp*[l]{Lexicographical minimum of the tuples in $A$.}
        $\mathrm{append}(\mathcal{P}, Q + \alpha\times(R - Q))$\;
    }
    $\mathrm{append}(\mathcal{P}, R)$\;
    \returnKw{$\mathcal{P}$} \tcp*[l]{Interpret the list $\mathcal{P}$ as a tuple and return it.}
    \caption{$\exactline{\mathrm{MaxPoolWindow}}{QR}$. Binary operations
    involving both scalars and vectors apply the operation element-wise to each
    component of the vector.}
    \label{alg:maxpool}
\end{algorithm}

\begin{proof}
    MaxPool applies a separate mapping from each input window to each output
    component, so it suffices to consider each window separately.

    Within a given window, the MaxPool operation returns the value of the
    maximum component. It is thus linear while the index of the maximum
    component remains constant. Now, we parameterize $\overline{QR}(\alpha) = Q
    + \alpha\times(R - Q)$. At each iteration of the loop we solve for the next
    point at which the maximum index changes. Assuming the maximum index is
    $m$ when $\alpha = \alpha_m$, we can solve for the next ratio $\alpha_i >
    \alpha_m$ at which index $i$ will become larger than $m$ like so (again
    realizing that linear functions are monotonic):
    \begin{align*}
            &\overline{QR}_m(\alpha_i) = \overline{QR}_i(\alpha_i) \\
            \implies &Q_m + \alpha_i\times(R_m - Q_m) = Q_i + \alpha_i\times(R_i - Q_i) \\
            \implies &\alpha_i\times(R_m - Q_m + Q_i - R_i) = Q_i - Q_m \\
            \implies &\alpha_i = \frac{Q_i - Q_m}{(R_m - Q_m) + Q_i - R_i}
    \end{align*}
    If $\alpha_m \leq \alpha_i < 1$, then component $i$ becomes larger than
    component $m$ at $\overline{QR}(\alpha_i)$. We can compute this for all
    other indices (producing set $A$ in the algorithm) then find the first
    index that becomes larger than $m$. We assign this index to $m$ and its
    ratio to $\alpha$. If no such index exists, we can conclude that $m$
    remains the maximum until $R$, thus additional endpoints are not needed.

    Thus, within any two points in $P$ the maximum component stays the same, so
    the MaxPool can be exactly replaced with a linear map returning only that
    maximum element.
\end{proof}

At worst, then, for each of the $w$ windows each of size $s$, we may add $O(s)$
new endpoints ($\overline{QR}(\alpha)$ is monotonic in each component so the
maximum index can only change $s$ times), and for each of those $O(s)$ new
endpoints we must re-compute $D$, which requires $\mathrm{O}(s)$ operations.
Thus, the time complexity for each window is $\mathrm{O}(s^2)$ and for the
entire MaxPool computation is $\mathrm{O}(ws^2)$.

Although most practical applications (especially on medium-sized networks) do
not reach that worst-case bound, on extremely large (ImageNet-sized) networks
we have found that such MaxPool computations end up taking the majority of the
computation time. We believe this is an area for future work, perhaps using
GPUs or other deep-learning hardware to perform the analysis.

\section{$\problemname$ for MaxPool + ReLU Layers}
\label{app:relumaxpool-exactline}
When a MaxPool layer is followed by a ReLU layer (or vice-versa), the preceding
algorithm may include a large number of unnecessary points (for example, if the
maximum index changes but the actual value remains less than $0$, the ReLU
layer will ensure that both pieces of the MaxPool output are mapped to the same
constant value $0$). To avoid this, the MaxPool algorithm above can be modified
to check before adding each $P_i$ whether the value at the maximum index is
below $0$ and thus avoid adding unnecessary points. This can be made slightly
more efficient by ``skipping'' straight to the first index where the value
becomes positive, but overall the worst-case time complexity remains
$\mathrm{O}(ws^2)$.

\section{$\problemname$ for General Piecewise-Linear Layers}
\label{app:pwl-exactline}

A more general algorithm can be devised for any piecewise-linear layer, as long
as the input space can be partitioned into finitely-many (possibly unbounded)
convex polytopes, where the function is affine within each one. For example,
\textsc{ReLU} fits this definition where the convex polytopes are the orthants.
Once this has been established, then, we take \emph{the union of the
hyperplanes defining the faces of each convex polytope}. In the \textsc{ReLU}
example, each convex polytope defining the linear regions corresponds to a
single orthant. Each orthant has an ``H-representation'' in the form $\{ x \mid
x_1 \leq 0 \wedge x_2 > 0 \wedge \ldots \wedge x_n \leq 0 \}$, where we say the
corresponding ``hyperplanes defining the faces'' of this polytope are $\{ \{ x
\mid x_1 = 0 \}, \ldots, \{ x \mid x_n = 0 \} \}$ (i.e., replacing the
inequalities in the conjunction with equalities).  Finally, given line segment
$\overline{QR}$, we compare $Q$ and $R$ to each hyperplane individually;
wherever $Q$ and $R$ lie on opposite sides of the hyperplane, we add the
intersection point of the hyperplane with $\overline{QR}$. Sorting the
resulting points gives us a valid $\exactline{f}{QR}$ tuple. If desired, the
minimization described in~\pref{sec:Primitive} can be applied to recover the
unique smallest $\exactline{f}{QR}$ tuple.

The intuition behind this algorithm is exactly the same as that behind the
\textsc{ReLU} algorithm; partition the line such that each resulting segment
lies entirely within a single one of the polytopes. The further intuition here
is that, if a point lies on a particular side of \emph{all of the faces}
defining the set of polytopes, then it must lie entirely within \emph{a single
one} of those polytopes (assuming the polytopes partition the input space).

Note that this algorithm can also be used to compute $\problemname$ for
\textsc{MaxPool} layers, however, in comparison, the algorithm
in~\pref{app:maxpool-exactline} effectively adds two optimizations. First, the
``search space'' of possibly-intersected faces at any point is restricted to
only the faces of the polytope that the last-added point resides in (minimizing
redundancy and computation needed). Second, we always add the first (i.e.,
closest to $Q$) intersection found, so we do not have to sort the points at the
end (we literally ``follow the line''). Such function-specific optimizations
tend to be beneficial when the partitioning of the input space is more complex
(eg.  \textsc{MaxPool}); for component-wise functions like \textsc{ReLU}, the
general algorithm presented above is extremely efficient.

\section{$\problemname$ Over Multiple Layers}
\label{app:multi-layer}

Here we prove~\pref{thm:multi-layer}, which formalizes the intuition that we can
solve $\problemname$ for an entire piecewise-linear network by solving it for
all of the intermediate layers individually. Then, we use $\problemname$ on
each layer in sequence, with each layer partitioning the input line segment
further so that $\problemname$ on each latter layer can be computed on each of
those partitions.

\ThmMultiLayer*

\begin{proof}
    Consider any linear partition of $g$ defined by endpoints $(P_i, P_{i+1})$
    of $\overline{QR}$. By the definition of $\exactline{g}{QR}$,
    there exists some affine map $A_i$ such that $g(x) = A_i(x)$ for any $x \in
    \overline{P_iP_{i+1}}$.

    Now, consider $\exactline{h}{g(P_i)g(P_{i+1})} = (O^i_1 = P_i, O^i_2,
    \ldots, O^i_m = P_{i+1})$. By the definition of
    $\exactline{h}{g(P_i)g(P_{i+1})}$, then, for any partition
    $\overline{O^i_jO^i_{j+1}}$, there exists some affine map $B^i_j$ such that
    $h(x) = B^i_j(x)$ for all $x \in \overline{O^i_jO^i_{j+1}}$.

    Realizing that $O^i_j, O^i_{j+1} \in \overline{g(P_i)g(P_{i+1})}$ and that
    $\overline{P_iP_{i+1}}$ maps to $\overline{g(P_i)g(P_{i+1})}$ under $g$
    (affineness of $g$ over $\overline{P_iP_{i+1}}$), and assuming $g(P_i) \neq
    g(P_{i+1})$ (i.e., $A_i$ is non-degenerate), there exist unique $I^i_j,
    I^i_{j+1} \in \overline{P_iP_{i+1}}$ such that $g(I^i_j) = O^i_j$ and
    $g(I^i_{j+1}) = O^i_{j+1}$. In particular, as affine maps retain ratios
    along lines, we have that:
    \[ I^i_j = P_i + \frac{O^i_j - g(P_i)}{g(P_{i+1}) - g(P_i)}\times (P_{i+1} - P_i)
    \]
    And similar for $I^i_{j+1}$. (In the degenerate case, we can take $I^i_j =
    P_i, I^i_{j+1} = P_{i+1}$ to maintain the partitioning).

    Now, we consider the line segment $\overline{I^i_jI^i_{j+1}} \subseteq
    \overline{P_iP_{i+1}}$. As it is a subset of $\overline{P_iP_{i+1}}$, all
    points $x \in \overline{I^i_jI^i_{j+1}}$ are transformed to $A_i(x) \in
    \overline{g(O^i_j)g(O^i_{j+1})}$ by $g$. Thus, the application of $h$ to
    any such point $y = A_i(x) \in \overline{O^i_jO^i_{j+1}}$ is $B^i_j(y)$,
    and the composition $(B^i_j \circ A_i)(x)$ is an affine map taking points
    $x \in \overline{I^i_jI^i_{j+1}}$ to $f(x)$.

    Finally, as the $O_j$s partition each $\overline{g(P_i)g(P_{i+1})}$ and
    each $\overline{P_iP_{i+1}}$ partitions $\overline{QR}$, and we picked
    $I^i_j$s to partition each $\overline{P_iP_{i+1}}$, the set of $I^i_j$s
    partitions $\overline{QR}$.
\end{proof}

This theorem can be applied for each layer in a network, allowing us to
identify a linear partitioning for the entire network with only linear
partitioning algorithms for each individual layer.

\section{Constant Gradients with $\problemname$}
\label{app:constgrad}

In~\pref{sec:IntegratedGradients}, we relied on the fact that the gradients are
constant within any linear partition given by $\exactline{\textsc{ReLU}}{QR}$
computed with~\pref{eq:alg}. This fact was formalized
by~\pref{thm:relu-constgrad}, which we prove below:

\ThmReluConstgrad*

\begin{proof}
    We first notice that the gradient of the entire network, when it is
    well-defined, is determined completely by the signs of the internal
    activations (as they control the action of the ReLU function).

    Thus, as long as the signs of the internal activations are constant, the
    gradient will be constant as well.

    Equation 4 in our paper identifies partitions where the signs of the
    internal activations are constant. Therefore, the gradient in each of those
    regions $\overline{P_iP_{i+1}}$ is also constant.
\end{proof}

However, \emph{in general}, for arbitrary $f$, it is possible that the action
of $f$ may be affine over the line segment $\overline{P_iP_{i+1}}$ but not
affine (or not describable by a single $A_i$) when considering points
arbitrarily close to (but not lying on) $\overline{P_iP_{i+1}}$. In other
words, the definition of $\exactline{f}{QR}$ as presented in our paper only
requires that the \emph{directional derivative} in the direction of
$\overline{QR}$ is constant within each linear partition
$\overline{P_iP_{i+1}}$, not the gradient more generally. A stronger definition
of $\problemname$ could integrate such a requirement, but we present the
weaker, more-general definition in the text for clarity of exposition.

However, as demonstrated in the above theorem, this stronger requirement
\emph{is} met by~\pref{eq:alg}, thus our exact computation of Integrated
Gradients is correct.

\section{Further $\problemname$ Implementation Details}
\label{app:implementation}

We implemented our algorithms for computing $\exactline{f}{QR}$ in C++ using a
gRPC server with Protobuf interface. This server can be called by a
fully-fledged Python interface which allows one to define or import networks
from ONNX~\cite{ONNXNetworks} or ERAN~\cite{ERAN} formats and compute
$\exactline{f}{QR}$ for them. For the ACAS Xu experiments, we converted the
ACAS Xu models provided in the ReluPlex~\cite{reluplex:CAV2017} repository to
ERAN format for analysis by our tool.

Internally, we represent $\exactline{f}{QR}$ by a vector of endpoints, each
with a \emph{ratio} along $\overline{QR}$ (i.e., $\alpha$ for the
parameterization $\overline{QR}(\alpha) = Q + \alpha\times(R - Q)$), the layer
at which the endpoint was introduced, and the corresponding post-image after
applying $f$ (or however many layers have been applied so far).

On extremely large (ImageNet-sized) networks, storing the internal
network activations corresponding to each of the thousands of endpoints
requires significant memory usage (often hundreds of gigabytes), so our
implementation sub-divides $\overline{QR}$ when necessary to control memory
usage.
However, for all tested networks, our implementation was extremely fast and
could compute $\exactline{f}{QR}$ for all experiments in seconds.

\section{Floating Point Computations}
\label{app:floatingpoint}

As with ReluPlex \cite{reluplex:CAV2017}, we make use of floating-point
computations in all of our implementations, meaning there may be some slight
inaccuracies in our computations of each $P_i$. However, where float
inaccuracies in ReluPlex correspond to hard-to-interpret errors relating to
pivots of simplex tableaus, floating point errors in our algorithms are easier
to interpret, corresponding to slight miscomputation in the exact position of
each linear partition endpoint $P_i$. In practice, we have found that these
errors are small and unlikely to cause meaningful issues with the use of
$\problemname$.  With that said, improving accuracy of results while retaining
performance is a major area of future work for both ReluPlex and
$\problemname$.

\section{Future Work}
\label{sec:FutureWork}
Apart from the future work described previously, $\problemname$ itself can be
further generalized. For example, while our algorithm is extremely fast (a
number of seconds) on medium- and large-sized convolutional and feed-forward
networks using ReLU non-linearities, it currently takes over an hour to execute
on large ImageNet networks due to the presence of extremely large MaxPool
layers. Scaling the algorithm and implementation (perhaps by using GPUs for
verification, modifying the architecture of the network itself, or involving
safe over-approximations) is an exciting focus for future work. Furthermore, we
plan to investigate the use of safe linear over-approximations to commonly used
non-piecewise-linear activation functions (such as $\tanh$) to analyze a wider
variety of networks.  Finally, we can generalize $\problemname$ to compute
restrictions of networks to higher-dimensional input regions, which may allow
the investigation of even more novel questions.

\end{document}